\documentclass[a4paper]{article}

\usepackage{authblk}
\usepackage{amsmath}
\usepackage{amssymb}
\usepackage{amsthm}
\usepackage{bbm}
\usepackage{graphicx}
\usepackage{hyperref}
\usepackage[round]{natbib}
\usepackage{url}

\theoremstyle{plain}
\newtheorem{assumption}{Assumption}
\newtheorem{lemma}{Lemma}
\newtheorem{theorem}{Theorem}

\theoremstyle{definition}
\newtheorem{definition}{Definition}

\title{Demystifying ResNet}

\author[1]{Sihan Li}
\author[2]{Jiantao Jiao}
\author[2]{Yanjun Han}
\author[2]{Tsachy Weissman}

\affil[1]{Tsinghua University}
\affil[2]{Stanford University}

\begin{document}

\maketitle

\begin{abstract}
The Residual Network (ResNet), proposed in~\citet{he2015deep}, utilized shortcut connections to significantly reduce the difficulty of training, which resulted in great performance boosts in terms of both training and generalization error.

It was empirically observed in~\citet{he2015deep} that stacking more layers of residual blocks with shortcut $2$ results in smaller training error, while it is not true for shortcut of length $1$ or $3$. We provide a theoretical explanation for the uniqueness of shortcut $2$.

We show that with or without nonlinearities, by adding shortcuts that have depth two, the condition number of the Hessian of the loss function at the zero initial point is depth-invariant, which makes training very deep models no more difficult than shallow ones. Shortcuts of higher depth result in an extremely flat (high-order) stationary point initially, from which the optimization algorithm is hard to escape. The shortcut $1$, however, is essentially equivalent to no shortcuts, which has a condition number exploding to infinity as the number of layers grows. We further argue that as the number of layers tends to infinity, it suffices to only look at the loss function at the zero initial point.

Extensive experiments are provided accompanying our theoretical results. We show that initializing the network to small weights with shortcut $2$ achieves significantly better results than random Gaussian (Xavier) initialization, orthogonal initialization, and shortcuts of deeper depth, from various perspectives ranging from final loss, learning dynamics and stability, to the behavior of the Hessian along the learning process.
\end{abstract}

\section{Introduction}

Residual network (ResNet) was first proposed in~\citet{he2015deep} and extended in~\citet{he2016identity}.
It followed a principled approach to add shortcut connections every two layers to a VGG-style network~\citep{simonyan2014very}. The new network becomes easier to train, and achieves both lower training and test errors. Using the new structure, \citet{he2015deep} managed to train a network with 1001 layers, which was virtually impossible before.
Unlike Highway Network~\citep{srivastava2015training,srivastava2015highway} which not only has shortcut paths but also borrows the idea of gates from LSTM~\citep{sainath2015convolutional}, ResNet does not have gates.
Later \citet{he2016identity} found that by keeping a clean shortcut path, residual networks will perform even better.

Many attempts have been made to improve ResNet to a further extent.
``ResNet in ResNet''~\citep{targ2016resnet} adds more convolution layers and data paths to each layer, making it capable of representing several types of residual units.
``ResNets of ResNets''~\citep{zhang2016residual} construct multi-level shortcut connections, which means there exist shortcuts that skip multiple residual units.
Wide Residual Networks~\citep{zagoruyko2016wide} makes the residual network shorter but wider, and achieves state of the art results on several datasets while using a shallower network.
Moreover, some existing models are also reported to be improved by shortcut connections, including Inception-v4~\citep{szegedy2016inception}, in which shortcut connections make the deep network easier to train.

Understanding why the shortcut connections in ResNet could help reduce the training difficulty is an important question. Indeed, \citet{he2015deep} suggests that layers in residual networks are learning residual mappings, making them easier to represent identity mappings, which prevents the networks from degradation when the depths of the networks increase.
However, \citet{veit2016residual} claims that ResNets are actually ensembles of shallow networks, which means they do not solve the problem of training deep networks completely.  In \citet{hardt2016identity}, they showed that for deep linear residual networks with shortcut $1$ does not have spurious local minimum, and analyzed and experimented with a new ResNet architecture with shortcut $2$.

We would like to emphasize that it is not true that every type of identity mapping and shortcut works. Quoting \citet{he2015deep}:

\vspace{.5ex}
\begin{center}
\parbox{.9\textwidth}{~~\emph{ \small ``But if $\mathcal{F}$ has only a single layer, Eqn.(1) is similar to a linear layer:
$y = W_1 x + x$, for which we have not observed advantages.'' }}
\end{center}

\begin{center}
\parbox{.9\textwidth}{~~\emph{ \small ``Deeper non-bottleneck ResNets (e.g., Fig.~5 left) also gain accuracy
from increased depth (as shown on CIFAR-10), but are not as economical
as the bottleneck ResNets. So the usage of bottleneck designs is mainly due
to practical considerations. We further note that the degradation problem
of plain nets is also witnessed for the bottleneck designs. '' }}
\end{center}
\vspace{1.5ex}

Their empirical observations are inspiring. First, the shortcut $1$ mentioned in the first paragraph do not work. It clearly contradicts the theory in~\citet{hardt2016identity}, which forces us to conclude that the nonlinear network behaves essentially in a different manner from the linear network. Second, noting that the \emph{non-bottleneck} ResNets have shortcut $2$, but the bottleneck ResNets use shortcut $3$, one sees that shortcuts with depth three also do not ease the optimization difficulties.

In light of these empirical observations, it is sensible to say that a reasonable theoretical explanation must be able to distinguish shortcut $2$ from shortcuts of other depths, and clearly demonstrate why shortcut $2$ is special and is able to ease the optimization process so significantly for deep models, while shortcuts of other depths may not do the job. Moreover, analyzing deep linear models may not be able to provide the right intuitions.

\section{Main results}

We provide a theoretical explanation for the unique role of shortcut of length $2$. Our arguments can be decomposed into two parts.

\begin{enumerate}
\item For very deep (general) ResNet, it suffices to initialize the weights at zero and search locally: in other words, there exist a global minimum whose weight functions for each layer have vanishing norm as the number of layers tends to infinity.
\item For very deep (general) ResNet, the loss function at the zero initial point exhibits radically different behavior for shortcuts of different lengths. In particular, the Hessian at the zero initial point for the $1$-shortcut network has condition number growing unboundedly when the number of layers grows, while the $2$-shortcut network enjoys a depth-invariant condition number. ResNet with shortcut length larger than $2$ has the zero initial point as a high order saddle point (with Hessian a zero matrix), which may be difficult to escape from in general.
\end{enumerate}

We provide extensive experiments validating our theoretical arguments. It is mathematically surprising to us that although the deep linear residual networks with shortcut $1$ has no spurious local minimum~\citep{hardt2016identity}, this result does not generalize to the nonlinear case and the training difficulty is not reduced. Deep residual network of shortcut length $2$ admits spurious local minimum in general (such as the zero initial point), but proves to work in practice.

As a side product, our experiments reveal that \emph{orthogonal initialization~\citep{saxe2013exact} is suboptimal}. Although better than Xavier initialization~\citep{glorot2010understanding}, the initial condition numbers of the networks still explode as the networks become deeper, which means the networks are still initialized on ``bad'' submanifolds that are hard to optimize using gradient descent.

\section{Model}

We first generalize a linear network by adding shortcuts to it to make it a \textit{linear residual network}.
We organize the network into $R$ \textit{residual units}.
The $r$-th residual unit consists of $L_r$ layers whose weights are $W^{r,1}, \ldots, W^{r,L_r-1}$, denoted as the \textit{transformation path}, as well as a shortcut $S^r$ connecting from the first layer to the last one, denoted as the \textit{shortcut path}.
The input-output mapping can be written as

\begin{equation}
    y = \prod_{r = 1}^R (\prod_{l = 1}^{L_r - 1} W^{r,l}  + S^r) x = W x,
\end{equation}

where $x \in \mathbb{R}^{d_x}, y \in \mathbb{R}^{d_y}, W \in \mathbb{R}^{d_y \times d_x}$.
Here if $b \ge a$, $\prod_{i=a}^b W^i$ denotes $W^b W^{(b-1)} \cdots W^{(a+1)} W^a$, otherwise it denotes an identity mapping.
The matrix $W$ represents the combination of all the linear transformations in the network.
Note that by setting all the shortcuts to zeros, the network will go back to a $(\sum_r (L_r - 1) + 1)$-layer plain linear network.

Instead of analyzing the general form, we concentrate on a special kind of linear residual networks, where all the residual units are the same.

\begin{definition}
    A linear residual network is called an \textit{$n$-shortcut linear network} if
    \begin{enumerate}
        \item its layers have the same dimension (so that $d_x = d_y$);
        \item its shortcuts are identity matrices;
        \item its shortcuts have the same depth $n$.
    \end{enumerate}
    The input-output mapping for such a network becomes

    \begin{equation}
        y = \prod_{r = 1}^R (\prod_{l = 1}^n W^{r,l}  + I_{d_x}) x = W x,
    \end{equation}

    where $W^{r,l} \in \mathbb{R}^{d_x \times d_x}$.

\end{definition}

Then we add some activation functions to the networks.
We concentrate on the case where activation functions are on the transformation paths, which is also the case in the latest ResNet~\citep{he2016identity}.

\begin{definition}
    An $n$-shortcut linear network becomes an \textit{$n$-shortcut network} if element-wise activation functions $\sigma_{\mathrm{pre}}(x), \sigma_{\mathrm{mid}}(x), \sigma_{\mathrm{post}}(x)$ are added at the transformation paths, where on a transformation path, $\sigma_{\mathrm{pre}}(x)$ is added before the first weight matrix, $\sigma_{\mathrm{mid}}(x)$ is added between two weight matrices and $\sigma_{\mathrm{post}}(x)$ is added after the last weight matrix.
\end{definition}

\begin{figure}[htbp]
    \centering
    \includegraphics[width=0.5\columnwidth]{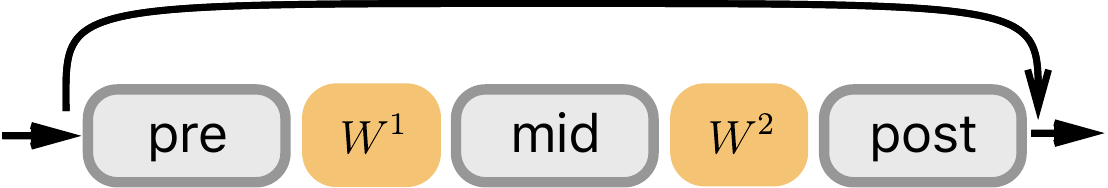}
    \caption{An example of different position for nonlinearities in a residual unit of a $2$-shortcut network.}
    \label{fig:nonlinearities-positions}
\end{figure}

Note that $n$-shortcut linear networks are special cases of $n$-shortcut networks, where all the activation functions are identity mappings.

\section{Theoretical study}

\subsection{Small weights property of (near) global minimum}

ResNet uses MSRA initialization~\citep{he2015delving}. It is a kind of scaled Gaussian initialization that tries to keep the variances of signals along a transformation path, which is also the idea behind Xavier initialization~\citep{glorot2010understanding}.
However, because of the shortcut paths, the output variance of the entire network will actually explode as the network becomes deeper.
Batch normalization units partly solved this problem in ResNet, but still they cannot prevent the large output variance in a deep network.

A simple idea is to zero initialize all the weights, so that the output variances of residual units stay the same along the network.
It is worth noting that as found in~\citet{he2015deep}, the deeper ResNet has smaller magnitudes of layer responses.
This phenomenon has been confirmed in our experiments. As illustrated in Figure~\ref{fig:resnet-norm} and Figure~\ref{fig:weight-norms}, the deeper a residual network is, the small its average Frobenius norm of weight matrices is, both during the training process and when the training ends.
Also, \citet{hardt2016identity} proves that if all the weight matrices have small norms, a linear residual network with shortcut of length $1$ will have no critical points other than the global optimum.

\begin{figure}[htbp]
    \centering
    \includegraphics[width=0.7\columnwidth]{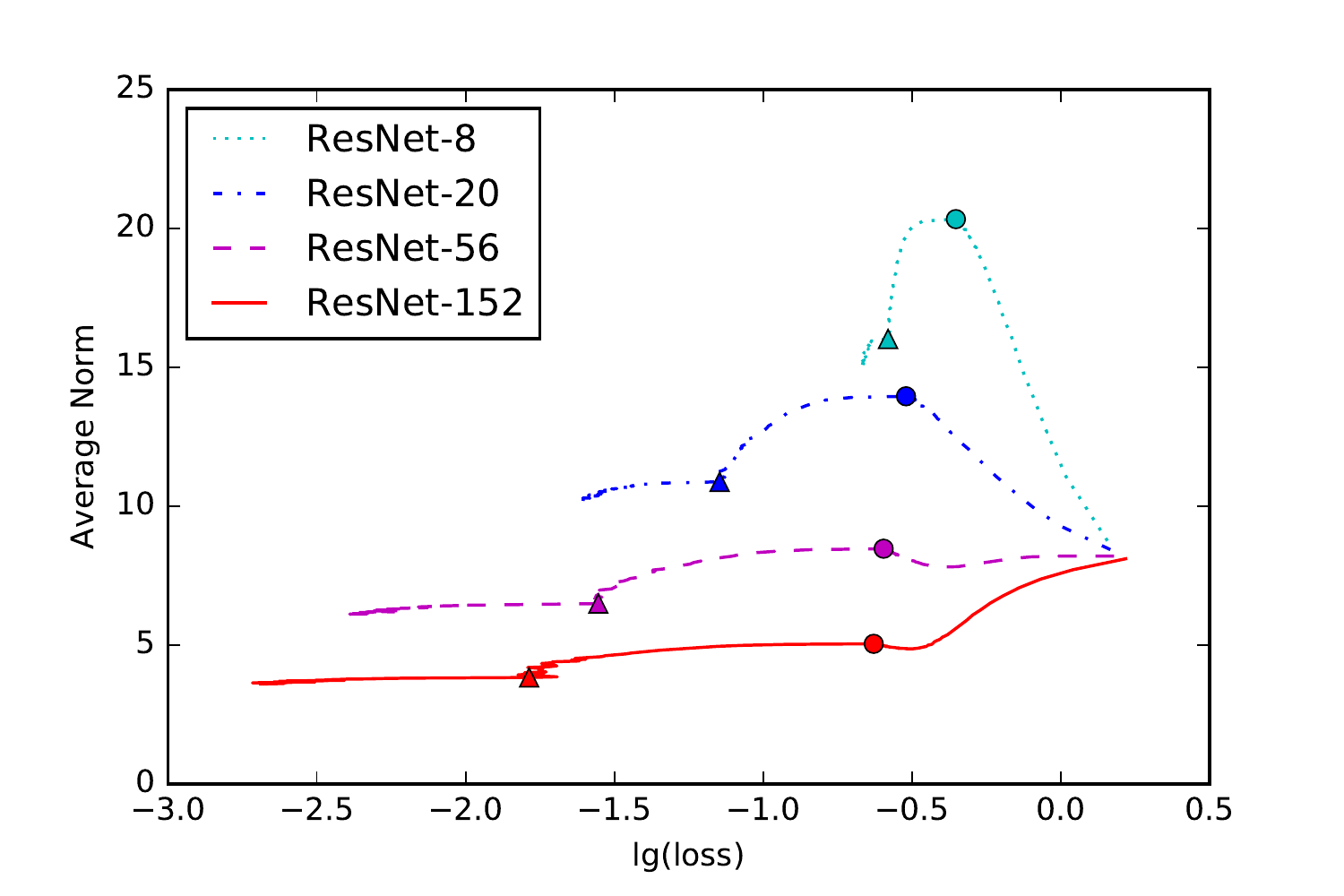}
    \caption{The average Frobenius norms of ResNets of different depths during the training process. The pre-ResNet implementation in \url{https://github.com/facebook/fb.resnet.torch} is used. The learning rate is initialized to 0.1, decreased to 0.01 at the 81\textsuperscript{st} epoch (marked with circles) and decreased to 0.001 at the 122\textsuperscript{nd} epoch (marked with triangles). Each model is trained for 200 epochs.}
    \label{fig:resnet-norm}
\end{figure}

\begin{figure}[htbp]
    \centering
    \includegraphics[width=1.0\columnwidth]{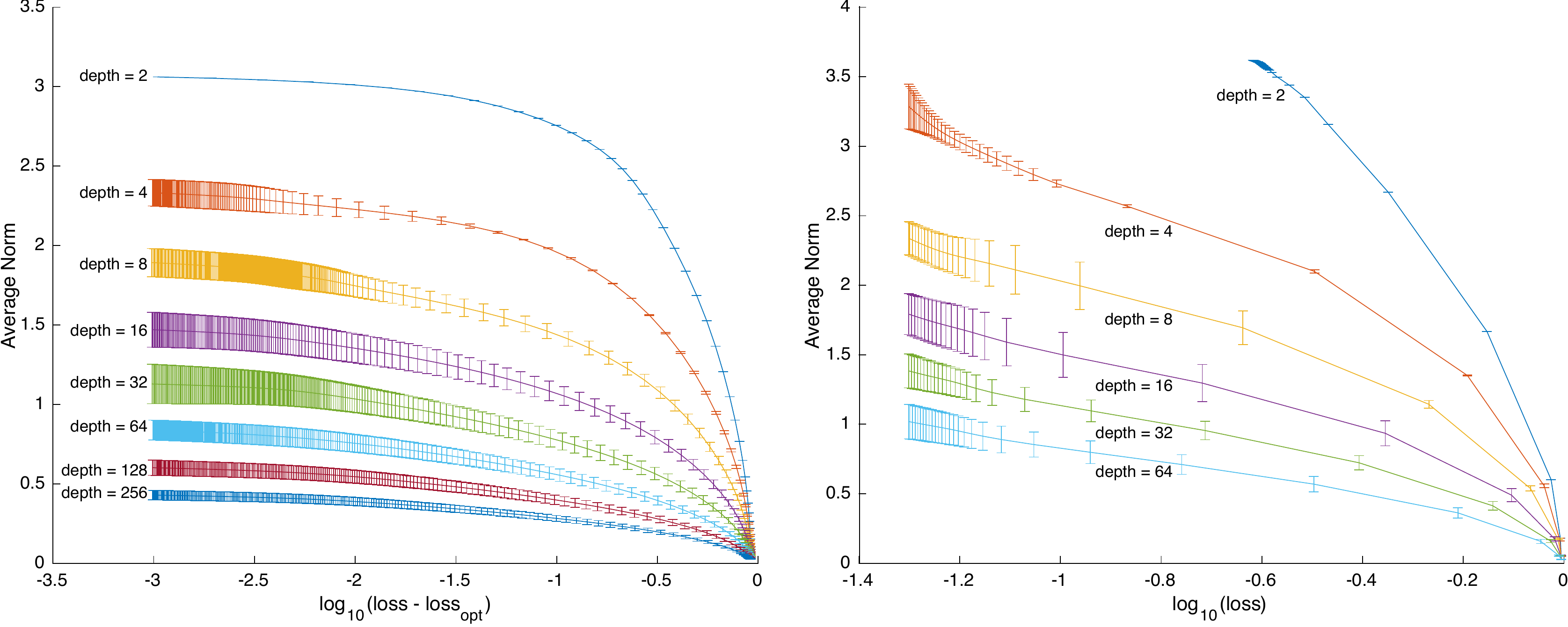}
    \caption{The average Frobenius norms of $2$-shortcut networks of different depths during the training process when zero initialized. \textbf{Left}: Without nonlinearities. \textbf{Right}: With ReLUs at mid positions. }
    \label{fig:weight-norms}
\end{figure}

All these evidences indicate that zero is special in a residual network: as the network becomes deeper, the training tends to end up around it.
Thus, we are looking into the Hessian at zero.
As the zero is a saddle point, in our experiments we use zero initialization with small random perturbations to escape from it. We first Xavier initialize the weight matrices, and then multiply a small constant ($0.01$) to them.

Now we present a simplified ResNet structure with shortcut of length $2$, and prove that as the residual network becomes deeper, there exists a solution whose weight functions have vanishing norm, which is observed in ResNet as we mentioned. This argument is motivated by~\citet{hardt2016identity}.

We concentrate on a special kind of network whose overall transformation can be written as

\begin{equation}
    \label{eq:simplified-resnet}
    y = \prod_{r = 1}^R (W^{r,2} \mathrm{ReLU}(W^{r,1} x + b_r) + I_{d_x}),
\end{equation}

where $b_r \in \mathbb{R}^{d_x}$ is the bias term.
It can seen as a simplified version of ResNet~\citep{he2016identity}.
Note that although this network is not a $2$-shortcut network, its Hessian still follow the form of Theorem~\ref{thm:cond}, thus its condition number is still depth-invariant.

We will also make some assumptions on the training samples.

\begin{assumption}
    \label{asm:training-samples-format}
    Assume $n \ge 3, m \gg 1$, for every $1 \le \mu \le m$, $\lVert x^\mu \rVert_2 = 1, y^\mu \in \{e_1, \cdots, e_{d_y}\}$ where $\{e_1, \cdots, e_{d_y}\}$ are $d_y$ standard basis vectors in $\mathbb{R}^{d_y}$.
\end{assumption}

The formats of training samples describe above are common in practice, where the input data are whitened and the labels are one-hot encoded.
Furthermore, we borrow an mild assumption from~\citet{hardt2016identity} that there exists a minimum distance between every two data points.

\begin{definition}
    The minimum distance of a group of vectors is defined as
    \begin{equation}
        d_{\text{min}}(a_1, a_2, \cdots, a_m) = \min_{1 \le i < j \le m}\lVert a_i - a_j \rVert_2,
    \end{equation}
    where $a_1, a_2, \cdots, a_m \in \mathbb{R}^d$.
\end{definition}

\begin{assumption}
    \label{asm:xs-min-distance}
    There exists a minimum distance $\rho$ between all the sample points and all the labels, i.e.
    \begin{equation}
        d_{\text{min}}(x^1, \cdots, x^m, y^1, \cdots, y^m) = \rho
    \end{equation}
\end{assumption}

As pointed out in~\citet{hardt2016identity}, this assumption can be satisfied by adding a small noise to the dataset. Given the model and the assumptions, we are ready to present our theorem whose proof can be found in Appendix~\ref{app:weight-norm-proof}.

\begin{theorem}
    \label{thm:weight-norm}
    Suppose the training samples satisfy Assumption~\ref{asm:training-samples-format} and Assumption~\ref{asm:xs-min-distance}. There exists a network in the form of Equation~\ref{eq:simplified-resnet} such that for every $1 \le r \le R, 1 \le l \le 2$,
    \begin{equation}
        \lVert W^{r,l} \rVert_F \le O(\sqrt{\frac{m}{n\rho}(m + \frac{1}{\rho})}).
    \end{equation}
\end{theorem}

For a specific dataset, $m, \rho$ are fixed, so the above equation can be simplified to
\begin{equation}
    \lVert W^{r,l} \rVert_F \le O(\sqrt{\frac{1}{n}}).
\end{equation}

This indicates that as the network become deeper, there exists a solution that is closer to the zero.
As a result, it is possible that in a zero initialized deep residual network, the weights are not far from the initial point throughout the training process, where the condition number is small, making it easy for gradient decent to optimize the network.

\subsection{Special properties of shortcut $2$ at zero initial point}

We begin with the definition of $k$-th order stationary point.

\begin{definition}
Suppose function $f(x)$ admits $k$-th order Taylor expansion at point $x_0$. We say that the point $x_0$ is a $k$-th order stationary point of $f(x)$ if the corresponding $k$-th order Taylor expansion of $f(x)$ at $x = x_0$ is a constant: $f(x) = f(x_0) + o(\| x - x_0\|_2^k)$.
\end{definition}

Then we make some assumptions on the activation functions.

\begin{assumption}
    \label{asm:activition-functions}
    $\sigma_{\mathrm{mid}}(0) = \sigma_{\mathrm{post}}(0) = 0$ and all of $\sigma_{\mathrm{pre}}^{(k)}(0), \sigma_{\mathrm{mid}}^{(k)}(0), \sigma_{\mathrm{post}}^{(k)}(0), 1\leq k \leq \max(n-1, 2)$ exist.
\end{assumption}

The assumptions hold for most activation functions including tanh, symmetric sigmoid and ReLU~\citep{nair2010rectified}. Note that although ReLU does not have derivatives at zero, one may do a local polynomial approximation to yield $\sigma^{(k)},1\leq k\leq \max(n-1, 2)$.

Now we state our main theorem, whose proof can be found in Appendix~\ref{app:cond-proof}.

\begin{theorem}
    \label{thm:cond}
    Suppose all the activation functions satisfy Assumption~\ref{asm:activition-functions}. For the loss function of an $n$-shortcut network, at point zero,
    \begin{enumerate}
        \item if $n \ge 2$, it is an $(n - 1)$th-order stationary point.
        In particular, if $n \ge 3$, the Hessian is a zero matrix;

        \item if $n = 2$, the Hessian can be written as
        \begin{equation}
            \label{eq:2-shortcut-hessian}
            H =
            \begin{bmatrix}
            \mathbf{0} & A^T \\
            A          & \mathbf{0} \\
                       &            & \mathbf{0} & A^T \\
                       &            & A          & \mathbf{0} \\
                       &            &            &            & \ddots
            \end{bmatrix},
        \end{equation}

        whose condition number is
        \begin{equation}
            \label{eq:2-shortcut-cond}
            \mathrm{cond}(H) = \sqrt{\mathrm{cond}((\Sigma^{X \sigma_{\mathrm{pre}}(X)} - \Sigma^{Y \sigma_{\mathrm{pre}}(X)})^T (\Sigma^{X \sigma_{\mathrm{pre}}(X)} - \Sigma^{Y \sigma_{\mathrm{pre}}(X)}))},
        \end{equation}

        where $A$ only depends on the training set and the activation functions.
        Except for degenerate cases, it is a \emph{strict} saddle point~\citep{ge2015escaping}.

        \item if $n = 1$, the Hessian can be written as

        \begin{equation}
            \label{eq:1-shortcut-hessian}
            H =
            \begin{bmatrix}
            B      & A^T    & A^T    & \cdots & A^T \\
            A      & B      & A^T    & \cdots & A^T \\
            A      & A      & B      &        & \vdots \\
            \vdots & \vdots &        & \ddots & A^T \\
            A      & A      & \cdots & A      & B
            \end{bmatrix}
        \end{equation}
        where $A, B$ only depend on the training set and the activation functions.
    \end{enumerate}
\end{theorem}

Theorem~\ref{thm:cond} shows that the condition numbers of $2$-shortcut networks are depth-invariant with a nice structure of eigenvalues. Indeed, the eigenvalues of the Hessian $H$ at the zero initial point are multiple copies of $\pm \sqrt{\mathrm{eigs}(A^T A)}$, and the number of copies is equal to the number of shortcut connections.

The Hessian at zero initial point for the $1$-shortcut network follows block Toeplitz structure, which has been well studied in the literature. In particular, its condition number tends to explode as the number of layers increase~\citep{gray2006toeplitz}.

To get intuitive explanations of the theorem, imagine changing parameters in an $n$-shortcut network.
One has to change at least $n$ parameters to make any difference in the loss.
So zero is an $(n - 1)$th-order stationary point.
Notice that the higher the order of a stationary point, the more difficult for a first order method to escape from it.

On the other hand, if $n = 2$, one will have to change two parameters in the same residual unit but different weight matrices to affect the loss, leading to a clear block diagonal Hessian.

\section{Experiments}

We compare networks with Xavier initialization~\citep{glorot2010understanding}, networks with orthogonal initialization~\citep{saxe2013exact} and $2$-shortcut networks with zero initialization.
The training dynamics of $1$-shortcut networks are similar to that of linear networks with orthogonal initialization in our experiments.
Setup details can be found in Appendix~\ref{app:exp-setup}.

\subsection{Initial point}

We first compute the initial condition numbers for different kinds of linear networks with different depths.

\begin{figure}[htbp]
    \centering
    \includegraphics[width=0.8\columnwidth]{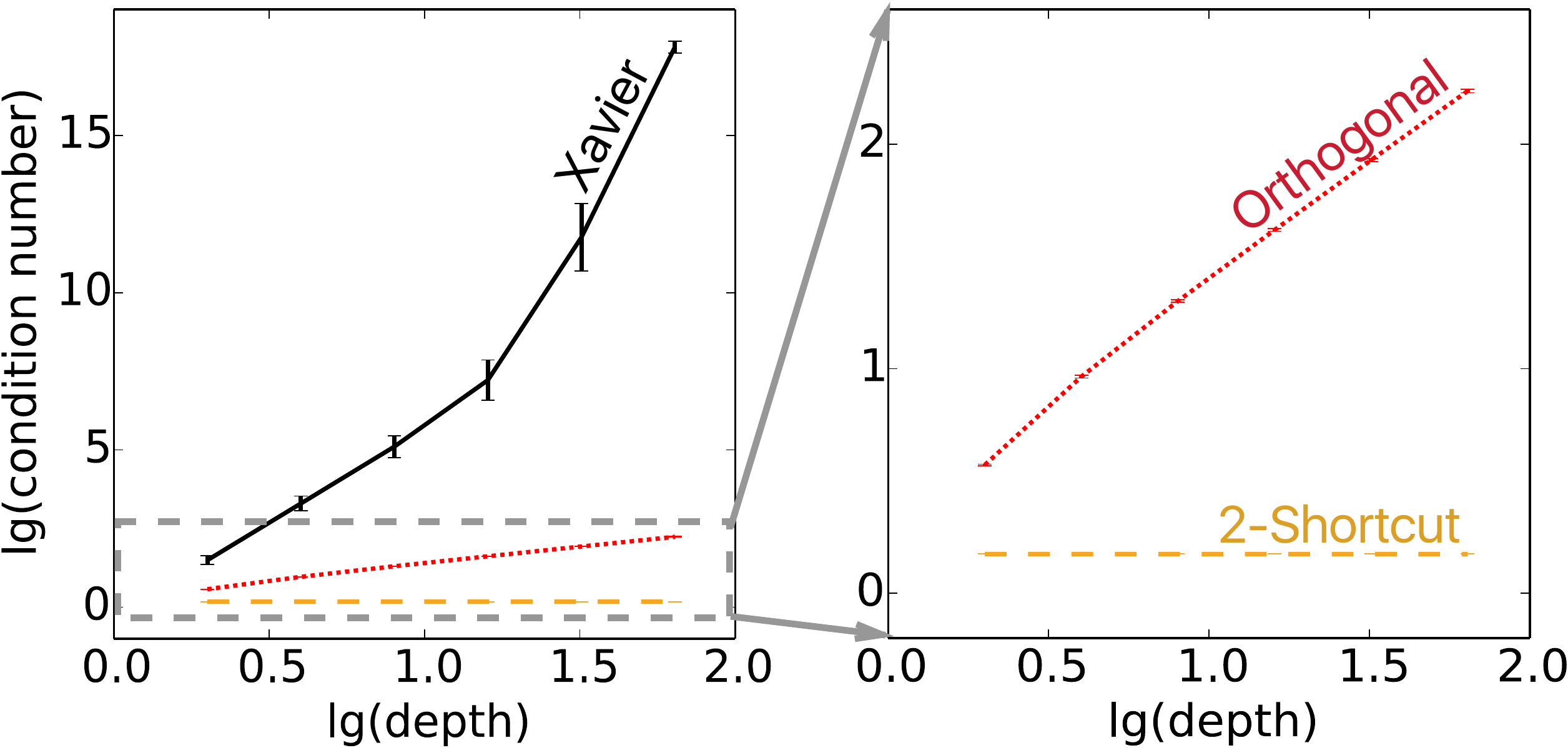}
    \caption{Initial condition numbers of Hessians for different linear networks as the depths of the networks increase. Means and standard deviations are estimated based on 10 runs.}
    \label{fig:cond-depth}
\end{figure}

As can be seen in Figure~\ref{fig:cond-depth}, \textbf{$2$-shortcut linear networks have constant condition numbers as expected.}
On the other hand, when using Xavier or orthogonal initialization in linear networks, the initial condition numbers will go to infinity as the depths become infinity, making the networks hard to train.
This also explains why orthogonal initialization is helpful for a linear network, as its initial condition number grows slower than the Xavier initialization.

\subsection{Learning dynamics}

Having a good beginning does not guarantee an easy trip on the loss surface. In order to depict the loss surfaces encountered from different initial points, we plot the maxima and 10\textsuperscript{th} percentiles (instead of minima, as they are very unstable) of the absolute values of Hessian’s eigenvalues at different losses.

\begin{figure}[htbp]
    \centering
    \includegraphics[width=1.0\columnwidth]{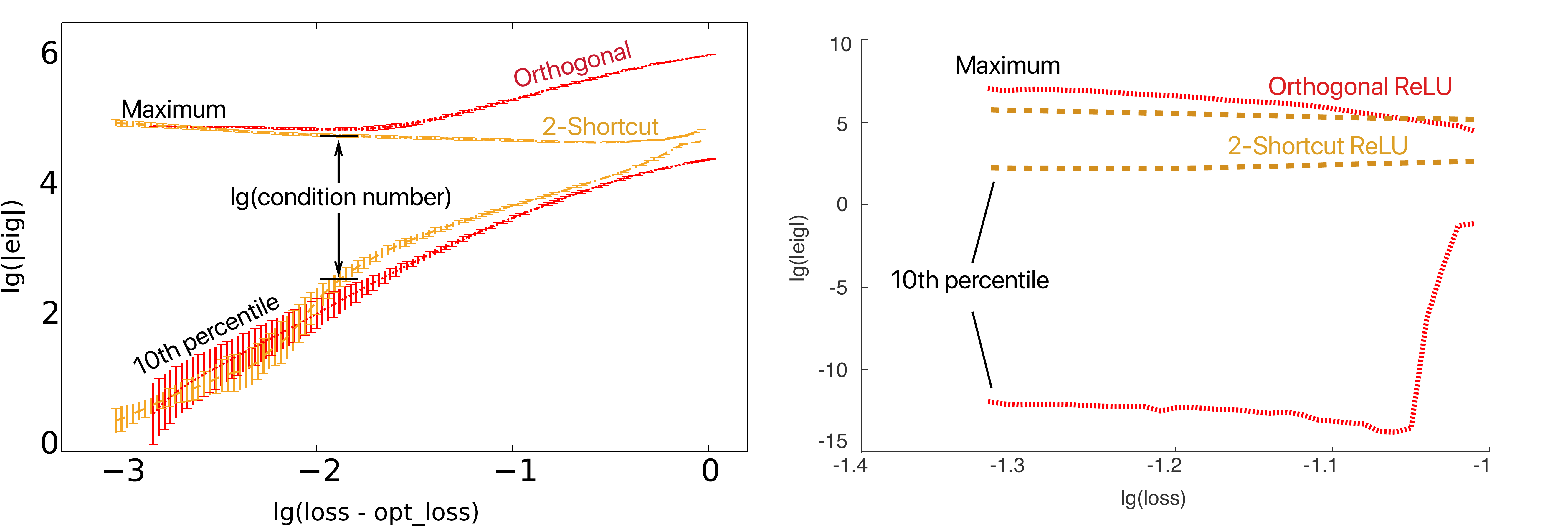}
    \caption{Maxima and 10\textsuperscript{th} percentiles of absolute values of eigenvalues at different losses when the depth is 16.
    For each run, eigenvalues at different losses are calculated using linear interpolation.}
    \label{fig:eigs}
\end{figure}

As shown in Figure~\ref{fig:eigs}, the condition numbers of $2$-shortcut networks at different losses are always smaller, especially when the loss is large.
Also, notice that the condition numbers roughly evolved to the same value for both orthogonal and $2$-shortcut linear networks. This may be explained by the fact that the minimizers, as well as any point near them, have similar condition numbers.

Another observation is the changes of negative eigenvalues ratios. \textit{Index} (ratio of negative eigenvalues) is an important characteristic of a critical point. Usually for the critical points of a neural network, the larger the loss the larger the index~\citep{dauphin2014identifying}. In our experiments, the index of a $2$-shortcut network is always smaller, and drops dramatically at the beginning, as shown in Figure~\ref{fig:index},~left. This might make the networks tend to stop at low critical points.

\begin{figure}[htbp]
    \centering
    \includegraphics[width=0.8\columnwidth]{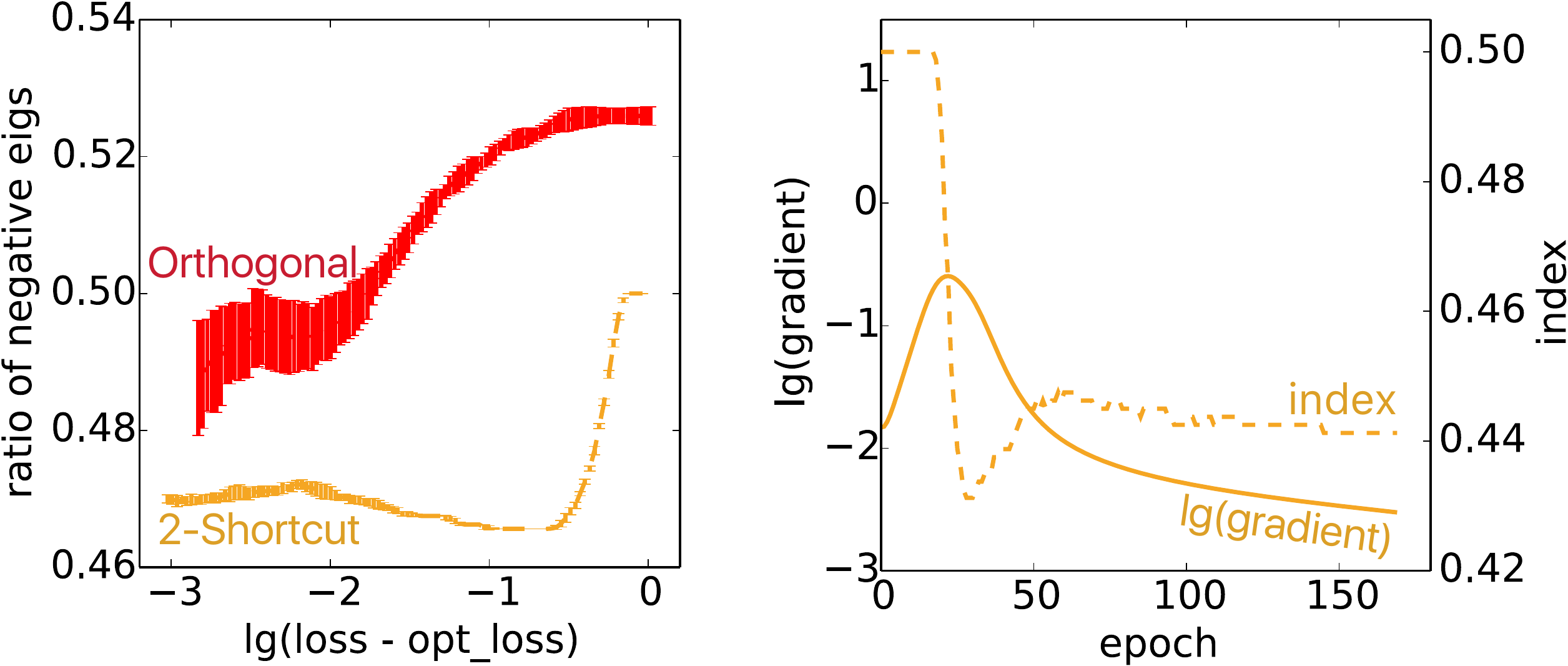}
    \caption{\textbf{Left}: ratio of negative eigenvalues at different losses when the depth is 16.
    For each run, indexes at different losses are calculated using linear interpolation.
    \textbf{Right}: the dynamics of gradient and index of a $2$-shortcut linear network in a single run. The gradient reaches its maximum while the index drops dramatically, indicating moving toward negative curvature directions.}
    \label{fig:index}
\end{figure}

This is because the initial point is near a saddle point, thus it tends to go towards negative curvature directions, eliminating some negative eigenvalues at the beginning. This phenomenon matches the observation that the gradient reaches its maximum when the index drops dramatically, as shown in Figure~\ref{fig:index},~right.

\subsection{Learning results}

\subsubsection{MNIST dataset}

We run different networks for 1000 epochs using different learning rates at log scale, and compare the average final losses corresponding to the optimal learning rates.

\begin{figure}[htbp]
    \centering
    \includegraphics[width=0.9\columnwidth]{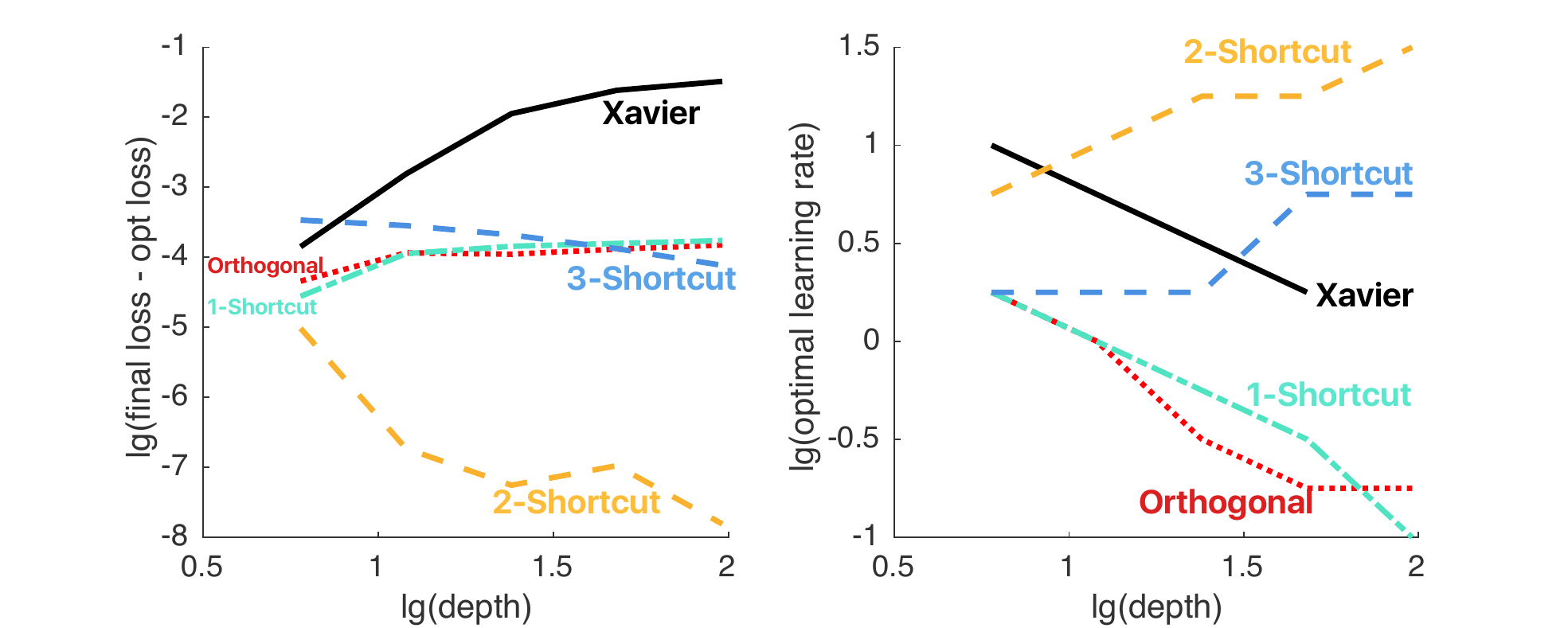}
    \caption{\textbf{Left}: Optimal final losses of different linear networks. \textbf{Right}: Corresponding optimal learning rates. When the depth is 96, the final losses of Xavier with different learning rates are basically the same, so the optimal learning rate is omitted as it is very unstable.}
    \label{fig:result}
\end{figure}

Figure~\ref{fig:result} shows the results for linear networks. Just like their depth-invariant initial condition numbers, the final losses of $2$-shortcut linear networks stay close to optimal as the networks become deeper. Higher learning rates can also be applied, resulting in fast learning in deep networks.

Then we add ReLUs to the \textit{mid} positions of the networks. To make a fair comparison, the numbers of ReLU units in different networks are the same when the depths are the same, so $1$-shortcut and $3$-shortcut networks are omitted. The result is shown in Figure~\ref{fig:relu-result}.

\begin{figure}[htbp]
    \centering
    \includegraphics[width=0.8\columnwidth]{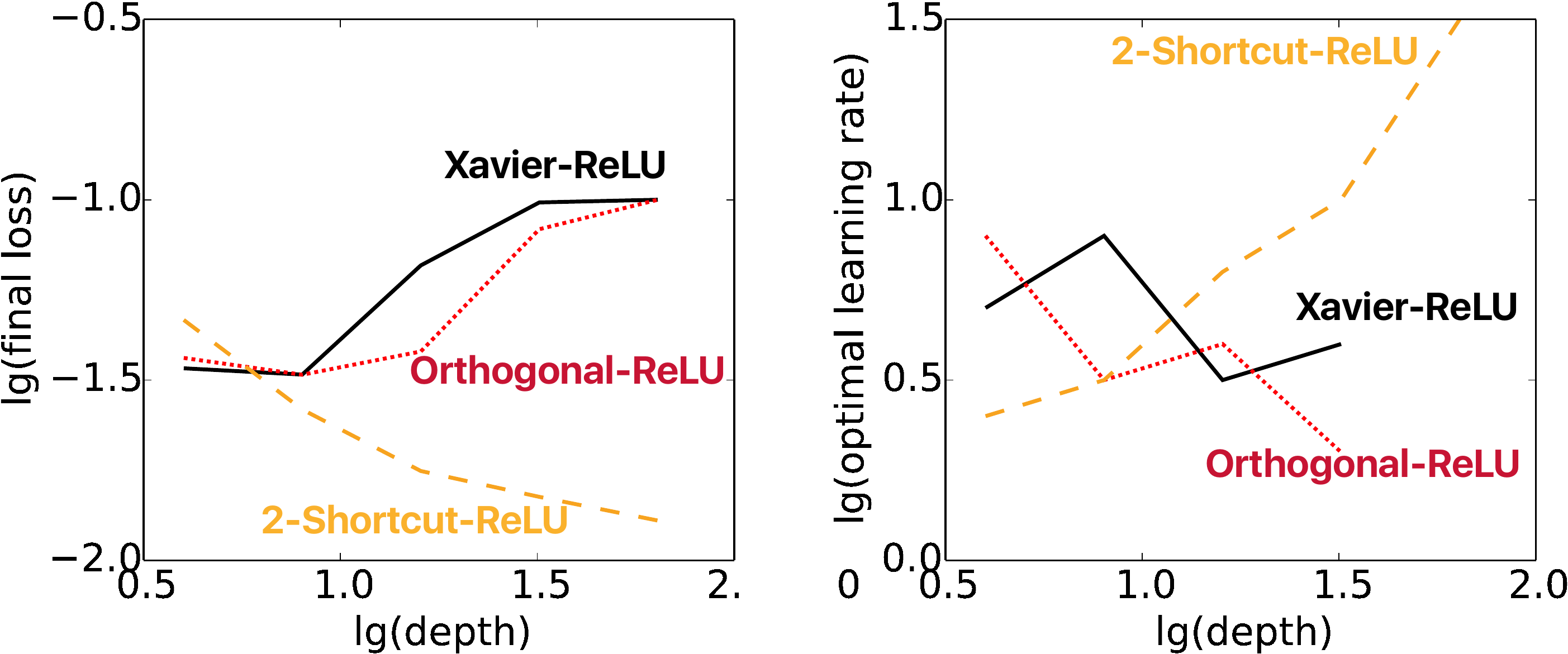}
    \caption{\textbf{Left}: Optimal final losses of different networks with ReLUs in \textit{mid} positions. \textbf{Right}: Corresponding optimal learning rates. Note that as it is hard to compute the minimum losses with ReLUs, we plot the $\log_{10}(\text{final loss})$ instead of $\log_{10}(\text{final loss} - \text{optimal loss})$. When the depth is 64, the final losses of Xavier-ReLU and orthogonal-ReLU with different learning rates are basically the same, so the optimal learning rates are omitted as they are very unstable.}
    \label{fig:relu-result}
\end{figure}

Note that because of the nonlinearities, the optimal losses vary for different networks with different depths. It is usually thought that deeper networks can represent more complex models, leading to smaller optimal losses. However, our experiments show that linear networks with Xavier or orthogonal initialization have difficulties finding these optimal points, while $2$-shortcut networks find these optimal points easily as they did without nonlinear units.

\subsubsection{CIFAR-10 dataset}

To show the effect of shortcut depth on a larger dataset, we modify the pre-ResNet implementation in \url{https://github.com/facebook/fb.resnet.torch} to make it possible to change shortcut depth while keeping the total number of parameters fixed. The default stopping criteria are used. The result is shown in Figure~\ref{fig:cifar10}. As shown in the figure, when the network becomes extremely deep ($>400$), only ResNets with shortcut $2$ gain advantages from the growth of depth, where other networks suffer from degradation as the network becomes deeper.

\begin{figure}[htbp]
    \centering
    \includegraphics[width=0.7\columnwidth]{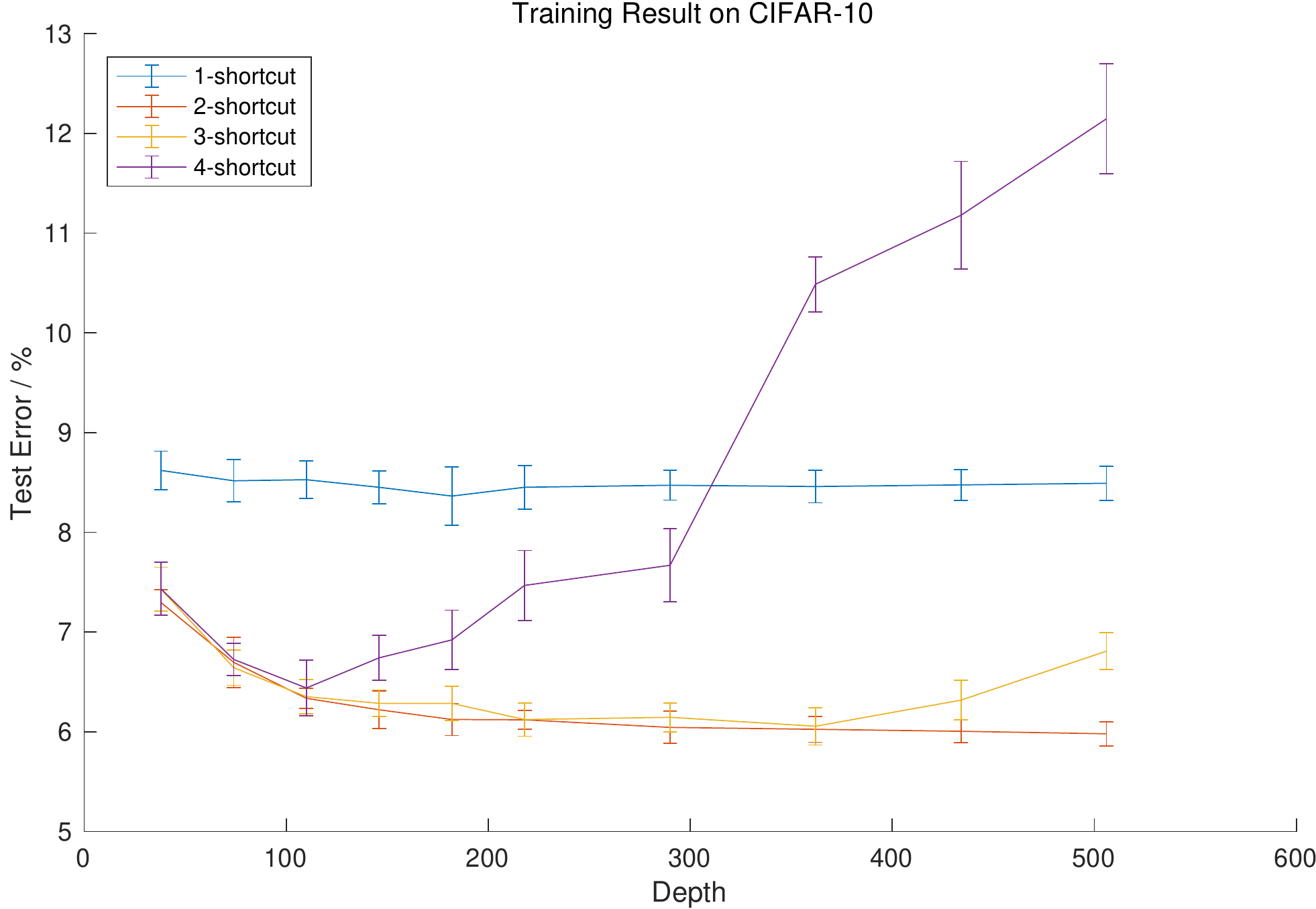}
    \caption{CIFAR-10 results of ResNets with different depths and shortcut depths. Means and standard deviations are estimated based on 10 runs. ResNets with a shortcut depth larger than 4 yield worse results and are omitted in the figure.}
    \label{fig:cifar10}
\end{figure}

\bibliographystyle{plainnat}
\bibliography{demystifying-resnet}

\appendix

\section{Proofs of theorems}

\subsection{Proof of Theorem~\ref{thm:weight-norm}}

\label{app:weight-norm-proof}
\begin{lemma}
    \label{lem:modify}
    Given matrix $A, A'$ such that
    \begin{align}
        \label{eq:matrix-a}
        A & = \begin{bmatrix}
            a_1 & \cdots & a_{i - 1} & a_i & a_{i + 1} \cdots & a_m
        \end{bmatrix}, \\
        \label{eq:matrix-b}
        A' & = \begin{bmatrix}
            a_1 & \cdots & a_{i - 1} & a_i' & a_{i + 1} & \cdots & a_m
        \end{bmatrix},
    \end{align}
    where $a_1, \cdots, a_i, \cdots, a_m$ and $a_i'$ are unit vectors in $\mathbb{R}^{d_x}$, $d_{\text{min}}(A), d_{\text{min}}(A) \ge \rho / 2$, $\lVert a_i - a_i' \rVert_2 = d$. There exists $W^1, W^2 \in \mathbb{R}^{d_x \times d_x}, b \in \mathbb{R}^{d_x}$ such that
    \begin{equation}
        \label{eq:residual-unit}
        W^2 \mathrm{ReLU}(W^1 A + b \cdot \mathbf{1}) + A = A',
    \end{equation}
    \begin{equation}
        \label{eq:min-distance}
        \lVert W^1 \rVert_F = \lVert W^2 \rVert_F = \frac{\sqrt{8d}}{\rho}.
    \end{equation}
\end{lemma}

\begin{proof}
    Let
    \begin{align}
        W^1 & = \frac{\sqrt{8d}}{\rho} \begin{bmatrix}
            a_i^T \\
            \mathbf{0} \\
            \vdots \\
            \mathbf{0}
        \end{bmatrix}, \\
        W^2 & = \frac{\sqrt{8/d}}{\rho} \begin{bmatrix}
            a_i' - a_i & \mathbf{0} & \cdots & \mathbf{0}
        \end{bmatrix}, \\
        b & =\frac{\sqrt{8d}}{\rho} (\frac{\rho^2}{8} - 1),
    \end{align}
    It is trivial to check that Equation~\ref{eq:residual-unit} and \ref{eq:min-distance} hold.
\end{proof}

Lemma~\ref{lem:modify} constructs a residual unit that change one column of its input by $d$.
Now we are going to proof that by repeating this step, the input matrix $X$ can be transfered into the output matrix $Y$.

\begin{lemma}
    \label{lem:steps}
    Given that $d_x \ge 3$, there exists a sequence of matrix $X^0, X^1, \cdots, X^s$ where
    \begin{equation}
        X^0 = X, X^{s} = Y,
    \end{equation}
    \begin{equation}
        s = m \lceil \frac{(\frac{\rho(m - 1)}{2} + 1) \pi}{d} \rceil,
    \end{equation}
    such that for every $1 \le i \le s$, $X^{i - 1}$ and $X^i$ conform to Lemma~\ref{lem:modify} with a distance smaller than $d$.
\end{lemma}

\begin{proof}
    In order to complete the transformation, we can modify $X$ column by column.
    For each column vector, in order to move it while preserving a minimum distance, we can draw a minor arc on the unit sphere connecting the starting and the ending point, bypassing each obstacle by a minor arc with a radius of $\rho / 2$ if needed, as shown in Figure~\ref{fig:sphere}. The length of the path is smaller than $(\frac{\rho(m - 1)}{2} + 1) \pi$, thus $\lceil \frac{(\frac{\rho(m - 1)}{2} + 1) \pi}{d} \rceil$ steps are sufficient to keep each step shorter than $d$. Repeating the process for $m$ times will give us a legal construction.
\end{proof}

\begin{figure}[htbp]
    \centering
    \includegraphics[width=0.3\columnwidth]{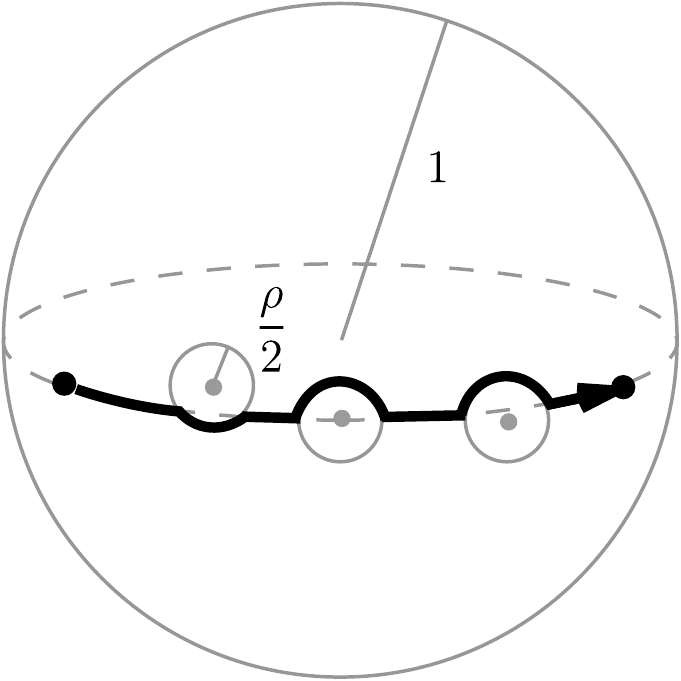}
    \caption{The path of a moving vector that preserves a minimum distance of $\rho / 2$.}
    \label{fig:sphere}
\end{figure}

Now we can prove Theorem~\ref{thm:weight-norm} with the all these lemmas above.

\begin{proof}[Proof of Theorem~\ref{thm:weight-norm}]
    Using Lemma~\ref{lem:steps}, we have $d \ge O(m(\rho m + 1))$. Then use Lemma~\ref{lem:modify}, we can get a construction that satisfies
    \begin{equation}
        \lVert W^{r,l} \rVert_F \le O(\sqrt{\frac{m}{n\rho}(m + \frac{1}{\rho})}).
    \end{equation}
\end{proof}

\subsection{Proof of Theorem~\ref{thm:cond}}
\label{app:cond-proof}
\begin{definition}
    The elements in Hessian of an $n$-shortcut network is defined as
    \begin{equation}
        H_{\mathrm{ind}(w_1), \mathrm{ind}(w_2)} = \frac{\partial^2 L}{\partial w_1 \partial w_2},
    \end{equation}
    where $L$ is the loss function, and the indices $\mathrm{ind}(\cdot)$ is ordered lexicographically following the four indices $(r, l, j, i)$ of the weight variable $w^{r,l}_{i,j}$. In other words, the priority decreases along the index of shortcuts, index of weight matrix inside shortcuts, index of column, and index of row.

\end{definition}

Note that the collection of all the weight variables in the $n$-shortcut network is denoted as $\mathbf{w}$. We study the behavior of the loss function in the vicinity of $\mathbf{w} = \mathbf{0}$.

\begin{lemma}
    \label{lem:zero}
    Assume that $w_1 = w^{r_1,l_1}_{i_1,j_1}, \cdots, w_N = w^{r_N,l_N}_{i_N,j_N}$ are $N$ parameters of an $n$-shortcut network. If $\frac{\partial^2 L}{\partial w_1 \cdots \partial w_N}\Big |_{\mathbf{w} = \mathbf{0}}$ is nonzero, there exists $r$ and $k_1, \cdots, k_n$ such that $r_{k_m} = r$ and $l_{k_m} = m$ for $m = 1, \cdots, n$.
\end{lemma}

\begin{proof}
    Assume there does not exist such $r$ and $k_1, \cdots, k_n$, then for all the shortcut units $r = 1, \cdots, R$, there exists a weight matrix $l$ such that none of $w_1, \cdots, w_N$ is in $W^{r,l}$, so all the transformation paths are zero, which means $W = I_{d_x}$. Then $\frac{\partial^2 L}{\partial w_1 \cdots \partial w_N}\Big |_{\mathbf{w} = \mathbf{0}} = 0$, leading to a contradiction.
\end{proof}

\begin{lemma}
    \label{lem:rearange}
    Assume that $w_1 = w^{r_1,l_1}_{i_1,j_1}, w_2 = w^{r_2,l_2}_{i_2,j_2}, r_1 \le r_2$.
    Let $L_0(w_1, w_2)$ denotes the loss function with all the parameters except $w_1$ and $w_2$ set to 0, $w'_1 = w^{1, l_1}_{i_1,j_1}, w'_2 = w^{1 + \mathbbm{1}(r_1 \neq r_2), l_2}_{i_2,j_2}$.
    Then $\frac{\partial^2 L_0(w_1, w_2)}{\partial w_1 \partial w_2}|_{(w_1,w_2) = \mathbf{0}} = \frac{\partial^2 L_0(w'_1, w'_2)}{\partial w'_1 \partial w'_2}|_{(w'_1, w'_2)=\mathbf{0}}$.
\end{lemma}

\begin{proof}
    As all the residual units expect unit $r_1$ and $r_2$ are identity transformations, reordering residual units while preserving the order of units $r_1$ and $r_2$ will not affect the overall transformation, i.e. $L_0(w_1, w_2)|_{w_1 = a, w_2 = b} = L'_0(w'_1, w'_2)|_{w'_1 = a, w'_2 = b}$. So $\frac{\partial^2 L_0(w_1, w_2)}{\partial w_1 \partial w_2}|_{(w_1,w_2) = \mathbf{0}} = \frac{\partial^2 L_0(w'_1, w'_2)}{\partial w'_1 \partial w'_2}|_{(w'_1, w'_2)=\mathbf{0}}$.
\end{proof}

\begin{proof}[Proof of Theorem~\ref{thm:cond}]
    Now we can prove Theorem~\ref{thm:cond} with the help of the previously established lemmas.
    \begin{enumerate}
        \item Using Lemma~\ref{lem:zero}, for an $n$-shortcut network, at zero, all the $k$-th order partial derivatives of the loss function are zero, where $k$ ranges from $1$ to $n-1$. Hence, the initial point zero is a $(n - 1)$th-order stationary point of the loss function.

        \item Consider the Hessian in $n = 2$ case.
        Using Lemma~\ref{lem:zero} and Lemma~\ref{lem:rearange}, the form of Hessian can be directly written as Equation~(\ref{eq:2-shortcut-hessian}), as illustrated in Figure~\ref{fig:block-hessian}.

        \begin{figure}[htbp]
            \centering
            \includegraphics[width=0.5\columnwidth]{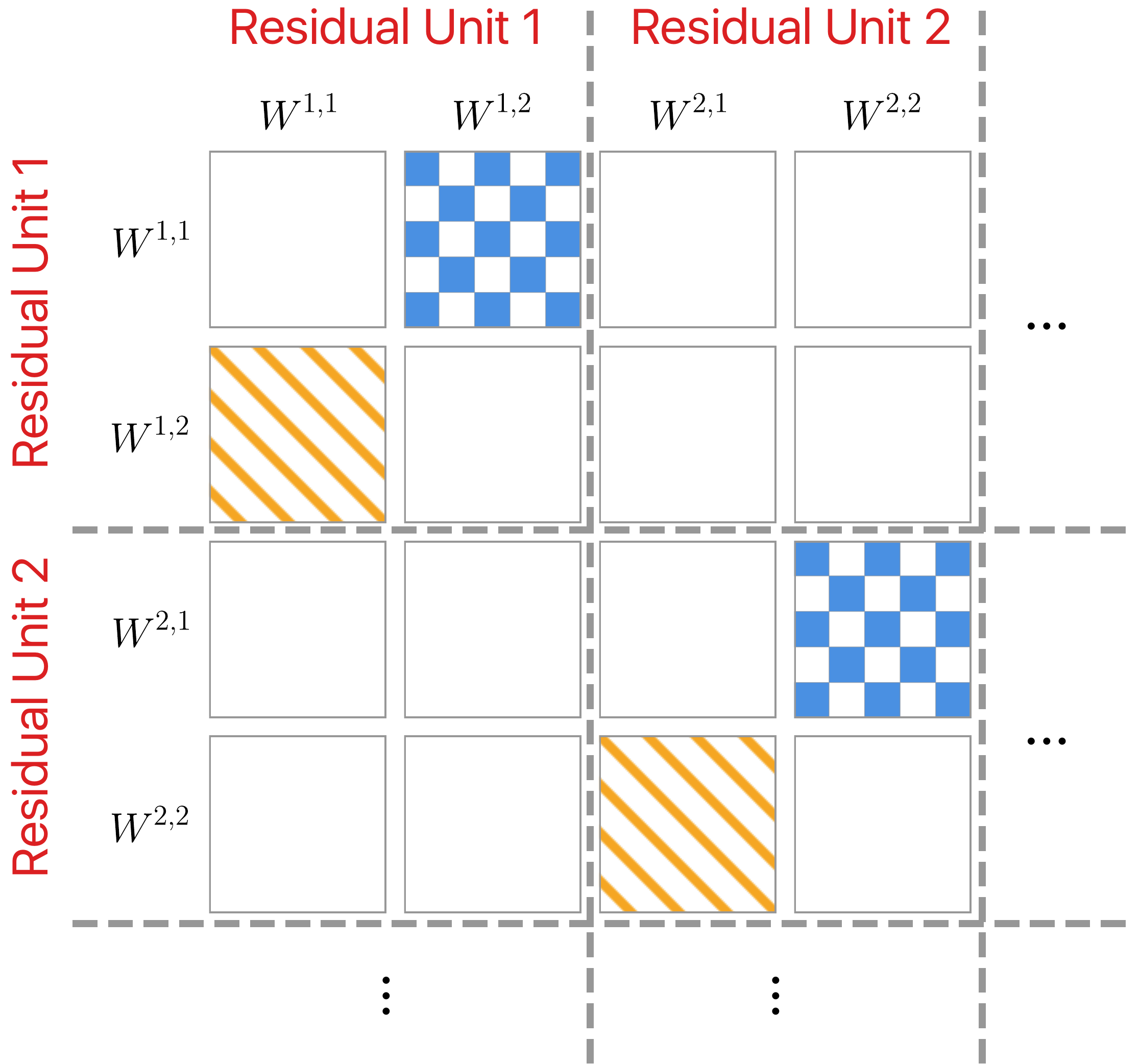}
            \caption{The Hessian in $n = 2$ case. It follows from Lemma~\ref{lem:zero} that only off-diagonal subblocks in each diagonal block, i.e., the blocks marked in orange (slash) and blue (chessboard), are non-zero. From Lemma~\ref{lem:rearange}, we conclude the translation invariance and that all blocks marked in orange (slash) (resp. blue (chessboard)) are the same. Given that the Hessian is symmetric, the blocks marked in blue and orange are transposes of each other, and thus it can be directly written as Equation~(\ref{eq:2-shortcut-hessian}).}
            \label{fig:block-hessian}
        \end{figure}

        So we have
        \begin{equation}
            \mathrm{eigs}(H) = \mathrm{eigs}(\begin{bmatrix}
                \mathbf{0} & A^T \\
                A          & \mathbf{0}
            \end{bmatrix}) = \pm \sqrt{\mathrm{eigs}(A^T A)}.
        \end{equation}
        Thus $\mathrm{cond}(H) = \sqrt{\mathrm{cond}(A^T A)}$, which is depth-invariant. Note that the dimension of $A$ is $d_x^2 \times d_x^2$.

        To get the expression of $A$, consider two parameters that are in the same residual unit but different weight matrices, i.e. $w_1 = w^{r,2}_{i_1,j_1}, w_2 = w^{r,1}_{i_2,j_2}$.

If $j_1 = i_2$, we have
        \begin{equation}
            \begin{split}
                A_{(j_1 - 1) d_x + i_1, (j_2 - 1) d_x + i_2} &= \frac{\partial^2 L}{\partial w_1 \partial w_2} \Big |_{\mathbf{w} = \mathbf{0}}\\
                &= \frac{\partial^2 \sum_{\mu = 1}^m  \frac{1}{2m} (y_{i_1}^\mu - x_{i_1}^\mu - \sigma_{\mathrm{post}}(w_1 \sigma_{\mathrm{mid}}(w_2 \sigma_{\mathrm{pre}}(x_{j_2}^\mu))))^2}{\partial w_1 \partial w_2}\Big |_{\mathbf{w} = \mathbf{0}} \\
                &= \frac{\sigma_{\mathrm{mid}}'(0) \sigma_{\mathrm{post}}'(0)}{m}\sum_{\mu = 1}^m \sigma_{\mathrm{pre}}(x_{j_2}^\mu) (x_{i_1}^\mu - y_{i_1}^\mu).
            \end{split}
        \end{equation}
Else, we have $A_{(j_1 - 1) d_x + i_1, (j_2 - 1) d_x + i_2} = 0$.

Noting that $A_{(j_1 - 1) d_x + i_1, (j_2 - 1) d_x + i_2}$ in fact only depends on the two indices $i_1,j_2$ (with a small difference depending on whether $j_1 = i_2$), we make a $d_x \times d_x$ matrix with rows indexed by $i_1$ and columns indexed by $j_2$, and the entry at $(i_1, j_2)$ equal to $A_{(j_1 - 1) d_x + i_1, (j_2 - 1) d_x + i_2}$. Apparently, this matrix is equal to $\sigma_{\mathrm{mid}}'(0) \sigma_{\mathrm{post}}'(0) (\Sigma^{X \sigma_{\mathrm{pre}}(X)} - \Sigma^{Y \sigma_{\mathrm{pre}}(X)})$ when $j_1 = i_2$, and equal to the zero matrix when $j_1 \neq i_2$.

        To simplify the expression of $A$, we rearrange the columns of $A$ by a permutation matrix, i.e.
        \begin{equation}
            A' = AP,
        \end{equation}
        where $P_{ij} = 1$ if and only if $i = ((j - 1) \bmod d_x) d_x + \lceil \frac{j}{d_x} \rceil$. Basically it permutes the $i$-th column of $A$ to the $j$-th column.

        Then we have
        \begin{equation}
            A = \sigma_{\mathrm{mid}}'(0) \sigma_{\mathrm{post}}'(0) \begin{bmatrix}
            \Sigma^{X \sigma_{\mathrm{pre}}(X)} - \Sigma^{Y \sigma_{\mathrm{pre}}(X)} \\
            & \ddots \\
            & & \Sigma^{X \sigma_{\mathrm{pre}}(X)} - \Sigma^{Y \sigma_{\mathrm{pre}}(X)}
            \end{bmatrix} P^T.
        \end{equation}

        So the eigenvalues of $H$ becomes
        \begin{equation}
            \mathrm{eigs}(H) = \pm \sigma_{\mathrm{mid}}'(0) \sigma_{\mathrm{post}}'(0) \sqrt{\mathrm{eigs}((\Sigma^{X \sigma_{\mathrm{pre}}(X)} - \Sigma^{Y \sigma_{\mathrm{pre}}(X)})^T (\Sigma^{X \sigma_{\mathrm{pre}}(X)} - \Sigma^{Y \sigma_{\mathrm{pre}}(X)}))},
        \end{equation}

        which leads to Equation~(\ref{eq:2-shortcut-cond}).

        \item Now consider the Hessian in the $n = 1$ case.
        Using Lemma~\ref{lem:rearange}, the form of Hessian can be directly written as Equation~(\ref{eq:1-shortcut-hessian}).

        To get the expressions of $A$ and $B$ in $\sigma_{\mathrm{pre}}(x) = \sigma_{\mathrm{post}}(x) = x$ case, consider two parameters that are in the same residual units, i.e. $w_1 = w^{r,1}_{i_1,j_1}, w_2 = w^{r,1}_{i_2,j_2}$.

We have
\begin{align}
B_{(j_1 - 1) d_x + i_1, (j_2 - 1) d_x + i_2} & = \frac{\partial^2 L}{\partial w_1 \partial w_2}\Big |_{\mathbf{w} = \mathbf{0}} \\
& = \begin{cases}  \frac{1}{m} \sum_{\mu = 1}^m x^\mu_{j_1} x^\mu_{j_2} & i_1 = i_2 \\ 0 & i_1 \neq i_2 \\ \end{cases}
\end{align}

        Rearrange the order of variables using $P$, we have
        \begin{equation}
            B = P \begin{bmatrix}
            \Sigma^{XX} \\
            & \ddots \\
            & & \Sigma^{XX}
            \end{bmatrix} P^T.
        \end{equation}

        Then consider two parameters that are in different residual units, i.e. $w_1 = w^{r_1,1}_{i_1,j_1}, w_2 = w^{r_2,1}_{i_2,j_2}, r_1 > r_2$.

We have
\begin{align}
A_{(j_1 - 1) d_x + i_1, (j_2 - 1) d_x + i_2} & = \frac{\partial^2 L}{\partial w_1 \partial w_2}\Big |_{\mathbf{w} = \mathbf{0}} \\
& = \begin{cases}
\frac{1}{m} \sum_{\mu = 1}^m (x^\mu_{i_1} - y^\mu_{i_1}) x^\mu_{j_2} + x^\mu_{j_1} x^{\mu}_{j_2} & j_1 = i_2, i_1 = i_2 \\
\frac{1}{m} \sum_{\mu = 1}^m (x^\mu_{i_1} - y^\mu_{i_1}) x^\mu_{j_2} & j_1 = i_2, i_1 \neq i_2 \\
\frac{1}{m} \sum_{\mu = 1}^m x^\mu_{j_1} x^\mu_{j_2}  & j_1 \neq i_2, i_1 = i_2 \\
0 & j_1 \ne i_2 , i_1 \ne i_2\\
\end{cases}
\end{align}

        In the same way, we can rewrite $A$ as
        \begin{equation}
            A = \begin{bmatrix}
            \Sigma^{XX}-\Sigma^{YX} \\
            & \ddots \\
            & & \Sigma^{XX}-\Sigma^{YX}
            \end{bmatrix}P^T + B.
        \end{equation}
    \end{enumerate}
\end{proof}

\section{Experiment setup on MNIST}
\label{app:exp-setup}

We took the experiments on whitened versions of MNIST.
Ten greatest principal components are kept for the dataset inputs.
The dataset outputs are represented using one-hot encoding.
The network was trained using gradient descent.
For every epoch, the Hessians of the networks were calculated using the method proposed in~\citet{bishop1992exact}.
As the $|\lambda|_{\text{min}}$ of Hessian is usually very unstable, we calculated $\frac{|\lambda|_{\text{max}}}{|\lambda|_{(0.1)}}$ to represent condition number instead, where $|\lambda|_{(0.1)}$ is the 10\textsuperscript{th} percentile of the absolute values of eigenvalues.

As \textit{pre}, \textit{mid} or \textit{post} positions are not defined in linear networks without shortcuts, when comparing Xavier or orthogonal initialized linear networks to $2$-shortcut networks, we added ReLUs at the same positions in linear networks as in $2$-shortcuts networks.

\end{document}